\documentclass[a4paper,UKenglish,cleveref, autoref, thm-restate]{lipics-v2021}

\pdfoutput=1 
\hideLIPIcs  

\title{A Diagrammatic Calculus for a Functional Model of Natural Language Semantics}

\titlerunning{String Diagrams of Semantics} 

\author{Matthieu Pierre Boyer}{DI ENS, Paris, France \and Department of Linguistics, Yale University, USA}{matthieu.boyer at ens.fr}{https://orcid.org/0000-0002-1825-0097}{}

\authorrunning{M. P. Boyer}

\Copyright{Matthieu P. Boyer}



\category{Student Paper} 

\relatedversion{} 

\acknowledgements{I want to thank my mother for help with the knitting vocab'
  and putting up with me asking many questions about terms in a language she
  does not speak while peeling potatos to try knitting with toothpicks; Antoine
  Groudiev for his precious insights on how to label equations which are not
  even presented in this paper; Paul-André Melliès for his insights on
  graphical languages and their use in diverse domains; Bob Frank and Bella
  Senturia for the help with the minimalistic merge syntactic theories; and
  \emph{last but not least}, Simon Charlow for his advising during my time at
  Yale, and his help around the linguistics questions that definitely arose and
  funding my stay there, both financially and spiritually.}

\nolinenumbers 

\EventEditors{John Q. Open and Joan R. Access}
\EventNoEds{2}
\EventLongTitle{42nd Conference on Very Important Topics (CVIT 2016)}
\EventShortTitle{CVIT 2016}
\EventAcronym{CVIT}
\EventYear{2016}
\EventDate{December 24--27, 2016}
\EventLocation{Little Whinging, United Kingdom}
\EventLogo{}
\SeriesVolume{42}
\ArticleNo{23}

\usepackage{bigstrut}
\usepackage{makecell}
\usepackage{tikz-dependency}
\usepackage{subcaption}
\usepackage{wrapfig}
\usepackage{xcolor}
\usepackage{calc}
\usepackage{graphicx}
\usetikzlibrary{decorations.markings, arrows.meta}
\usetikzlibrary{decorations.pathreplacing,calligraphy}
\usetikzlibrary{automata, arrows, calc, matrix, positioning, math}
\usetikzlibrary{intersections}
\pgfdeclarelayer{background}
\pgfsetlayers{background,main}
\tikzset{>=stealth}
\newcommand{\catstyle}[2]{
	\tikzset{#1/.style={color=#2}}
}
\tikzset{dot/.style={circle,draw=black,fill=black,minimum size=1mm,inner sep=0mm}}

\def\cont{\Gamma\vdash}
\def\poulpe{\qquad}

\DeclareMathOperator{\Var}{Var}

\def\ppl{\mathbin{+\mkern-12mu+}}

\definecolor{vulm}{HTML}{7d1dd3}
\definecolor{yulm}{HTML}{ffe500}

\usepackage{amsmath, amsfonts, amssymb, amsthm}
\usepackage{mathrsfs}
\usepackage{dsfont}
\usepackage{stmaryrd}
\usepackage{mathtools}
\usepackage{tikz-qtree}
\usepackage{tipa}
\usepackage{tikz-cd}
\usepackage{nicematrix}
\usepackage{contour}
\usepackage{multicol}
\contourlength{0.005em}

\def\ty#1{\ensuremath{\texttt{\color{yulm!70!black}#1}}}
\def\f#1{\ensuremath{\texttt{\color{vulm}#1}}}
\def\w#1{\ensuremath{\mathbf{#1}}\,}

\def\e{\ty{e}}
\def\t{\ty{t}}
\def\r{\ty{r}}

\renewcommand{\O}{\mathcal{O}}

\newcommand{\mL}{\mathcal{L}}

\newcommand{\mC}{\mathcal{C}}

\newcommand{\mF}{\mathcal{F}}

\renewcommand{\phi}{\varphi}
\renewcommand{\epsilon}{\varepsilon}

\newcommand{\id}{\mathrm{id}}

\newcommand{\suchthat}{\,\middle|\,}

\DeclareMathOperator{\Obj}{Obj}

\newcommand{\abs}[1]{\left|#1\right|}

\newcommand{\scalar}[1]{\left\langle #1 \right\rangle}

\newenvironment{mgrammar}%
{
	\setlength\tabcolsep{4pt}
	\begin{tabular}{>{$}l<{$}>{$}r<{$}>{$}l<{$} r}
		}
		{
	\end{tabular}
}

{
	\setlength\tabcolsep{4pt}
	\begin{tabular}{l r l r}
		}
		{
	\end{tabular}
}

\newcommand*{\firstrule}[3]{#1 &::= & #2& \quad \ifstrempty{#3}{}{\textit{(#3)}}\\}
\newcommand*{\lfrule}[3]{#1 &::= & #2& \quad \ifstrempty{#3}{}{\textit{(#3)}}}

\newcommand*{\gskip}{&&\\}

\newcommand\fracpush{\hfill\mbox{}}
\newcommand\fracnotate[1]{\fracpush\rlap{#1}}

\newcolumntype{C}{>{$}c<{$}}
	\newcolumntype{L}{>{$}l<{$}}
\newcolumntype{R}{>{$}r<{$}}
\def\fmap{\texttt{fmap}}

\newcommand{\suppfrac}[2]{%
	\sbox0{$\genfrac{}{}{0pt}{0}{\ensuremath{#1}}{\ensuremath{#2}}$}%
	\ooalign{%
		\hidewidth
		$\vcenter{\moveright\nulldelimiterspace
				\hbox to\wd0{%
					\xleaders\hbox{\kern.5pt\vrule height 0.4pt width 1.5pt\kern.5pt}\hfill
					\kern-1.5pt
				}%
			}$
		\hidewidth\cr
		$\genfrac{}{}{0pt}{0}{\raisebox{5pt}{\ensuremath{#1}}}{\ensuremath{#2}}$\cr}%
}

\def\combMR{\text{MR}}
\def\combML{\text{ML}}
\def\combUR{\text{UR}}
\def\combUL{\text{UL}}
\def\combC{\text{C}}
\def\combJ{\text{J}}
\def\combA{\text{A}}
\def\combER{\text{ER}}
\def\combEL{\text{EL}}
\def\combDN{\text{DN}}

\makeatletter
\newcommand{\@word}[4][]{%
	#2 & #3 & #4\\
	\ifx&#1&%
	\else
	&\multicolumn{2}{l}{Generalizes to \textbf{#1}}\\%
	\fi%
}
\def\word#1#2#3#4{\@word[#4]{#1}{#2}{#3}}
\makeatother

\usepackage{calc}

\makeatletter
\def\textSq#1{%
	\begingroup
	\setlength{\fboxsep}{0.4ex}
	\setbox1=\hbox{#1}
	\setlength{\@tempdima}{\maxof{\wd1}{\ht1+\dp1}}
	\setlength{\@tempdimb}{(\@tempdima-\ht1+\dp1)/2}
	\raise-\@tempdimb\hbox{\fbox{\vbox to \@tempdima{%
				\vfil\hbox to \@tempdima{\hfil\copy1\hfil}\vfil}}}%
	\endgroup%
}

\def\c@lsep{2.3}
\def\r@wsep{.5}
\def\wordsep{.2}

\tikzset{
	uptree/.style={
			draw=vulm!80!black,
			thick,
		},
	typenode/.style={
			align=center,
			text width=24mm,
		},
	treenode/.style={
			align=center,
			text width=24mm,
		},
	wordnode/.style={
			inner sep=0pt,
			align=center,
		},
	downtree/.style={
			draw=yulm!80!black,
			thick,
		},
}

\newcommand{\wnode}[3]{%
	\node (#2) at (#1*\c@lsep, 0) [wordnode] {#2};
	\node[anchor=north] (#2-) at ($(#1*\c@lsep, 0) + (0, -.142)$) [typenode] {\ensuremath{#3}};
}
\newcommand{\utnode}[3]{%
	\path let \p1 = (#2.north), \p2 = (#3.north) in coordinate (Q1) at (\x1, {max(\y1, \y2)});
	\path let \p1 = (#2.north), \p2 = (#3.north) in coordinate (Q2) at (\x2, {max(\y1, \y2)});
	\node (#2#3) at ($($(Q1)!0.5!(Q2)$) + (0, \r@wsep)$) [treenode] {\ensuremath{#1}};
	\draw[uptree] ($(#2) + (0, \wordsep)$) -- ($(#2#3) + (0, -\wordsep)$);
	\draw[uptree] ($(#3) + (0, \wordsep)$) -- ($(#2#3) + (0, -\wordsep)$);
}
\newcommand{\dtnode}[4][0.5]{%
	\path let \p1 = (#3.south), \p2 = (#4.south) in coordinate (Q1) at (\x1, {min(\y1, \y2)});
	\path let \p1 = (#3.south), \p2 = (#4.south) in coordinate (Q2) at (\x2, {min(\y1, \y2)});
	\node (#3#4) at ($($(Q1)!#1!(Q2)$) + (0, -\r@wsep)$) [treenode] {\ensuremath{#2}};
	\draw[downtree] ($(#3) + (0, -\wordsep)$) -- ($(#3#4) + (0, +\wordsep)$);
	\draw[downtree] ($(#4) + (0, -\wordsep)$) -- ($(#3#4) + (0, +\wordsep)$);
}

\makeatother

\catstyle{catone}{gray!50}
\catstyle{catmc}{vulm!10!yulm}
\catstyle{catmca}{vulm!20!yulm}
\catstyle{catmcb}{vulm!30!yulm}
\catstyle{catmcc}{vulm!40!yulm}
\catstyle{catmcd}{vulm!50!yulm}
\catstyle{catmce}{vulm!60!yulm}
\catstyle{catmcf}{vulm!70!yulm}
\catstyle{catmcg}{vulm!80!yulm}
\catstyle{catmch}{vulm!90!yulm}

\def\dlb#1#2{#1\mathrm{.L}\left(#2\right)}

\def\dnlg#1{#1\mathrm{.h}}
\def\dnin#1{#1\mathrm{.in}}
\def\dnout#1{#1\mathrm{.out}}

\tikzset{
	on each segment/.style={
			decorate,
			decoration={
					show path construction,
					moveto code={},
					lineto code={
							\path [#1]
							(\tikzinputsegmentfirst) -- (\tikzinputsegmentlast);
						},
					curveto code={
							\path [#1] (\tikzinputsegmentfirst)
							.. controls
							(\tikzinputsegmentsupporta) and (\tikzinputsegmentsupportb)
							..
							(\tikzinputsegmentlast);
						},
					closepath code={
							\path [#1]
							(\tikzinputsegmentfirst) -- (\tikzinputsegmentlast);
						},
				},
		},
	mid arrow/.style={postaction={decorate,decoration={
							markings,
							mark=at position .5 with {\arrow[#1]{stealth}}
						}}},
}

\usepackage{adjustbox}

\tikzcdset{scale cd/.style={every label/.append style={scale=#1},
			cells={nodes={scale=#1}}}}

\begin{document}

\maketitle

\begin{abstract}
	In this paper, we study a functional programming approach to natural language
	semantics, allowing us to increase the expressiveness of a more traditional
	denotation style.
	We will formalize a category based type and effect system to represent the
	semantic difference between syntactically equivalent expressions.
	We then construct a	diagrammatic calculus to model parsing and handling of
	effects, providing a method to efficiently compute the denotations for
	sentences.
\end{abstract}

\section{Introduction}
What is \emph{a chair}? How do I know that \emph{Jupiter, a planet}, is
\emph{a planet}?
To answer those questions, \cite{bumfordEffectdrivenInterpretationFunctors2025}
provide a \textsc{Haskell} based view on the notion of typing in natural
language semantics.
Their main idea is to include a layer of effects which allows for improvements
in the expressiveness of the denotations used.
This allows us to model complex concepts such as anaphoras, or non-determinism in
an easy way, independent of the actual way the words are represented.
Indeed, when considering the usual denotations of words as typed lambda-terms,
this allows us to solve the issue of meaning getting lost through impossible
typing, while still being able to compose meanings properly.
When two expressions have the same syntactic distribution, they must also have
the same type, which forces quantificational noun phrases to have the same type
as proper nouns: the entity type $\e$.
However, there is no singular entity that is the referent of \emph{every
	planet}, and so, the type system gets in the way of meaning, instead of
serving it.

\smallskip

Our formalism is inscribed in the contemporary natural language semantic
theories which are based on three main elements: a \emph{lexicon}, a
\emph{syntactic description} of the language, and a theory of
\emph{composition}.
More specifically, we explain how to extend the domain of the lexicon and the
theory of composition to account for the phenomena described above.
We will not be discussing most of the linguistic foundations for the usage of
the formalism, nor its usefulness.
We refer the reader to \cite{bumfordEffectdrivenInterpretationFunctors2025} to
get an overview of the linguistic considerations that are the base of the
theory.

\smallskip

In this paper, we will provide a formal definition of an enhanced type and
effect system for natural language semantics, based on categorical tools.
This will increase the complexity (both in terms of algorithmic operations and
in comprehension of the model) of the parsing algorithms, but through the use
of string diagrams to model the effect of composition on potential effectful
denotations (or more generally computations), we will provide efficient
algorithms for computing the set of meanings of a sentence, from the meaning of
its components.

\section{Related Work}
This is not the first time a categorical representation of compositional
semantics of natural language is proposed,
\cite{coeckeMathematicalFoundationsCompositional2010} already suggested an
approach based on monoidal categories using an external model of meaning.
What our approach gives more, is additional latitude for the definition of
denotations in the lexicon, and a visual explanation of the difference between
multiple possible parsing trees.
The proposition of \cite{toumiHigherOrderDisCoCatPeirceLambekMontague2023} is
closer to our proposition on graphical aspects, but still has the limits of
using an external model of meaning while ours expands on the use of an expanded
model of computation.
We will go back later on our more abstract way of looking at the semantic
parsing of a sentence.

\smallskip

On a completely different approach,
\cite{marcollimatildeetchomskynoametberwickrobertc.MathematicalStructureSyntactic}
provide a categorical structure based on Hopf algebra and coloured operads
to explain their model of syntax, leading to results at the interface of syntax
and morphology presented in \cite{senturiaAlgebraicStructureMorphosyntax2025}.
Similarly, \cite{melliesCategoricalContoursChomskySchutzenberger2025} provides
a modeling of CFGs using coloured operads.
Our approach is based on the suggestion that merge in syntax can be done using
labels, independent on how it is mathematically modelled.

\section{Categorical Semantics of Effects: A Typing System}
In this section, we will formalize a type system underlying the theory proposed
in \cite{bumfordEffectdrivenInterpretationFunctors2025}.
To do so, we will designate by $\mL$ our language, as a set of words
(with their associated meaning/denotation) and syntactic rules underlying
the semantic combination.
The absence of syntactic rules is allowed, although it partly defeats the
purpose of this work.
This might be useful when proposing compositional models of learned
representations.

We will use $\mF\left( \mL \right)$ to denote the set of functors or
higher-order functions used in denotations of $\mL$.
Those are chosen when representing the language (see Figure \ref{fig:functors}
for examples), and should be additions to a simpler semantic theory.
Our goals here are to describe more formally, using a categorical vocabulary,
the environment in which the typing system for our language will exist, and how
we connect words and other linguistic objects to the categorical formulation.

\subsection{Typing Category}
\subsubsection{Types}
Let $\mC$ be a closed cartesian category representing the
domain of types for the domains and co-domains of uneffectful denotations.
$\mC$ is our \emph{main} typing system, consisting of types for words
that can be expressed without effects (see Figure \ref{fig:lexicon} for an
example).
The terminal object $\bot$ of $\mC$ represents the empty type or the lack
thereof.
We consider as our typing category $\bar{\mC}$ the categorical closure for
exponentials and products of $\mF\left( \mL \right)^{*}\left(\mC\right)$,
which consists of all the different type constructors (ergo, functors) that
could be formed in the language.
In that setting our types are those that can be attained from a finite number
of functorial applications from an object of $\mC$.

Since $\mF\left( \mL \right)$ only induces a preorder on
$\Obj\left( \bar{\mC} \right)$, we consider the relation on types
$x\succeq y \Leftrightarrow \exists F, y = F(x)$ (which should be seen
as a subtyping relation as proposed in \cite{melliesFunctorsAreType2015}).
We then consider for our types the quotient set
$\star = \Obj\left(\bar{\mC}\right) / ~$ where $~$ is the transitive closure of
the subtyping relationship induced by functorial application.
We also define $\star_{0}$ to be the subset of types containing only
uneffectful types, i.e. $\Obj(\mC)$.
In contexts of polymorphism, we identify $\star_{0}$ to the adequate subset of
$\star$.
In this paradigm, constant objects (or results of fully handled computations) are
functions with type $\bot \to \tau$ which we will denote directly by
$\tau \in \star_{0}$.
This will be useful when defining base combinators in Section \ref{sec:parsing}.

\subsubsection{Functors, Applicatives and Monads}
Our point of view has us consider \emph{language functors}\footnote{Words with
	denotations in $\mF(\mL)$ which represent denotationally effectful
	constructions, e.g. "a" or "the". They are to be considered with opposition to
	the \emph{type functors} which are the mathematical construct in $\mF(\mL)$.}
as polymorphic functions: for a
(possibly restrained) set of base types $S$, a functor is seen as a function:
\begin{equation*}
	x: \tau\in S\subseteq \star \mapsto F x: F\tau
\end{equation*}
This means that if a functor can be applied to a type, it can also be applied
to all \emph{affected} versions of that type, i.e.
$\mF\left( L \right)(\tau\in \star)$.
This gives us two typing judgements for the functor $F$:
\begin{equation*}
	\frac{\Gamma\vdash x: \tau \in \star_{0}}{\Gamma\vdash F x: F\tau \notin
		\star_{0}} \hspace{2cm} \frac{\Gamma\vdash x:
		\tau}{\Gamma\vdash Fx : F\tau\preceq \tau}
\end{equation*}
We use the same notation for the \emph{language functor} and the
\emph{type functor} in the examples, but it is important to note those are two
different objects, although connected.
More precisely, the \emph{language functor} is to be seen as a function whose
computation yields an effect, while the \emph{type functor} is the endofunctor
of $\bar{\mC}$ (so a functor from $\mC$) that represents the effect in our
typing category.
Examples of this difference are to be found in Figures \ref{fig:lexicon} and
\ref{fig:functors}.

\smallskip

In this regard, applicatives and monads only provide with more flexibility on
the ways to combine functions:
they provide intermediate judgements to help with the combination of trees.
For example, the multiplication of the monad provides a new \emph{type
	transformation} judgement allowing derivation of $M\tau$ from$MM\tau$.
This is a special case of the natural transformation rule that we
define in the next section.

\subsubsection{Natural Transformations}
We could add judgements directly for adjunctions and monads, but we generalize
by adding judgements for natural transformations, as adjunctions and monadic
rules are natural transformations which arise from \emph{natural} settings.
While in general we do not want to create natural transformations, we want to be
able to express these in three situations:
\begin{enumerate}
	\item Adjunctions, Monads and Co-Monads\footnote{Which are actually the same
		      thing.}.
	\item To deal with the resolution of effects as explained in Section
	      \ref{sec:nondet}
	\item To create \emph{higher-order} constructs which transform words from our
	      language into other words, while keeping the functorial aspect.
	      This idea is developed in Section \ref{par:higherorder}.
\end{enumerate}
To see why we want this rule, which is a larger version of the monad
multiplication and the monad/applicative unit, it suffices to see that the
diagram defining the properties of a natural transformation provides a way
to construct the \emph{correct function} on the \emph{correct functor} side of
types.
From a linguistic point of view, natural transformations allow us to reason
directly about type coercions and their coherence in the typing system,
whether that is transporting effects across functors as in Section
\ref{par:higherorder} or collapsing nested effects and more generally handling
them as presented in Section \ref{par:handlers} and \ref{sec:nondet}.

\smallskip

In the Haskell programming language, any polymorphic function is
a natural transformation from the first type constructor to the second type
constructor, as proved in \cite{wadlerTheoremsFree1989}.
This will guarantee for us that given a \emph{Haskell} construction for a
polymorphic function, we will get the associated natural transformation.

\paragraph{Handlers}
\label{par:handlers}
As introduced by \cite{marsikAlgebraicEffectsHandlers}, the notion of handlers
is to be considered as the way to solve effects that obfuscate the result of a
computation.
Following \cite{wuEffectHandlersScope2014}, we understand handlers as natural
transformations describing the resolution of an algebraic effect: they are
natural transformations from the effect to the identity functor, effectively
resolving them.
Considering handlers this way allows us to directly handle our computations
inside our typing system and in particular inside our parsing algorithm.
This process will mostly be described in Sections \ref{sec:nondet} and
\ref{sec:parsing}.

\smallskip

To define a handler $h$, we will only require that for any applicative functor
of unit $\eta$, $h\circ \eta = \id$.
This solves the issue of non-termination of the system.
Note that the choice of the handler being part of the lexicon or the parser
over the other is a philosophical question more than a semantical one, as both
options will result in semantically equivalent models, the only difference will
be in the way we consider the resolution of effects.
This choice does not arise in the case of the adjunction-induced
handlers.
Indeed here, the choice is caused by the non-uniqueness of the choices for
the handlers as two different speakers may have different ways to resolve the
non-determinism effect that arises from the phrase \textsl{A chair}.
This is the difference with the adjunctions: adjunctions are intrinsic
properties of the coexistence of the effects, while the handlers
are user-defined.

\paragraph{Higher-Order Constructs}
\label{par:higherorder}
Functors may also be used to add plurals, superlatives, tenses, aspects and
other similar constructs which act as function modifiers.
For each of these, we give a functor $\Pi$ corresponding to a new class of
types along with natural transformations for all other functors $F$ which
allows to propagate down the high-order effect.
This allows us to add complexity not in the compositional aspects but
in the model of the language, by simply saying that those constructs are
predicate modifiers passed down (with or without side effects) to the arguments
of predicates:
\begin{equation*}
	\begin{aligned}
		\mathbf{future\left( be \right)\left( arg_{1}, arg_{2} \right)}
		 & \xrightarrow{\eta} \mathbf{future\left( be \right)\left( arg_{2} \right)\left( future\left( arg_{1} \right) \right)}                           \\
		 & \xrightarrow{\eta} \mathbf{future \left( be \right) \left( future \left( arg_{2} \right) \right) \left( future \left( arg_{1} \right) \right)}
	\end{aligned}
\end{equation*}

Among other higher-order constructs that might be represented using effects are
scope islands, which could be modelled by a functor that cannot be
passed as argument to words that would otherwise need a closure to be applied
first.
See Figure \ref{fig:tree-rain} for an example, based on theory presented in
\cite{bumfordEffectdrivenInterpretationFunctors2025}, Section 5.4.

The term ''\emph{higher-order construct}'' comes from the idea that those
constructs are not generated by words but at the scale of the sentence, or even
the syntax in the case of \emph{scope islands}.
As such, we will say that this type of functors are \emph{external} to the
lexicon.

\subsection{Typing Judgements}\label{subsec:judgements}
To complete this section, Figure \ref{tab:judgements} gives a simple list of different typing composition judgements through which we also re-derive the subtyping judgement to allow for its implementation.
\begin{figure}
	\input{typing-judgements}
	\caption{Typing and subtyping judgements for implementation of effects in the
		type system.}
	\label{tab:judgements}
\end{figure}
Note that here, the syntax is not taken into account: a function is always written left of its arguments, whether or not they are actually in that order in the sentence.

\smallskip

Using these typing rules for our semantic parsing steps, it is important to
see that our grammar will still bear ambiguity.
The next sections will explain how to reduce this ambiguity in short enough
time.

Moreover, our current typing system is not decidable, because of the
\texttt{nat/pure/return} rules which may allow for unbounded derivations.
This is not actually an issue because of the considerations on handling, as
semantically void units will get removed at that time.
Indeed, from the property of handlers adding a unit and not modifying the
effect before it is handled does not change anything to the result and will be
removed.
This leads to derivations of sentences to be of bounded height, linear in the
length of the sentence.

\section{Handling Ambiguity}
\label{sec:nondet}
The typing judgements proposed in Section \ref{subsec:judgements} lead to
ambiguity.
In this section we propose ways to get our derivations to a certain normal
form, by deriving an equivalence relation on our derivation and parsing trees,
based on string diagrams.

\subsection{String Diagram Modelisation of Sentences}
\label{subsec:sd}
String diagrams are the Poincaré duals of the usual categorical diagrams when
considered in the $2$-category of categories and functors.
This means that we represent categories as regions of the plane, functors as
lines separating regions and natural transformations as the intersection points
between two lines.

We will always consider application as applying to the right of the line so
that composition is written in the same way as in equations.
This gives us a new graphical formalism to represent our effects using a few
equality rules between diagrams.
The commutative aspect of functional diagrams is now replaced by an equality of
string diagrams, which will be detailed in the following section.

We get a way to visually see the meaning get reduced from effectful composition
to propositional values, without the need to specify what the handler does.
This delimits our usage of string diagrams as ways to look at computations and
a tool to provide equality rules to reduce ambiguity.

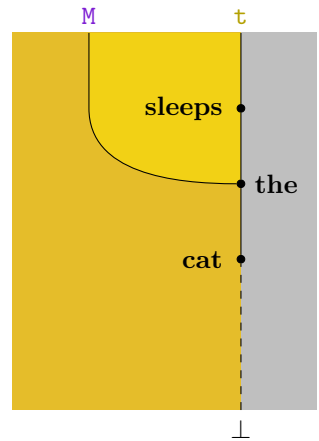
\begin{wrapfigure}[17]{r}{.45\textwidth}
	\centering
		\begin{tikzpicture}
		\path coordinate[dot, label=right:$\w{the}$] (the) + (0, 1) coordinate[dot, label=left:$\w{sleeps}$] (sleeps) + (0, 2) coordinate[label=above:$\t$] (bool)
		++(-2, 1) coordinate (ctlthe) + (0, 1) coordinate[label=above:$\f{M}$] (effthe)
		++(2, -2) coordinate[dot, label=left:$\w{cat}$] (cat) + (0, -2) coordinate[label=below:$\bot$] (bot);
		\draw (cat) -- (the) -- (sleeps) -- (bool);
		\draw[name path=effect] (the) to[out=180, in=-90] (ctlthe) -- (effthe);
		\draw[dashed] (bot) -- (cat);
		\begin{pgfonlayer}{background}
			\fill[catone] (bot) rectangle ($(bool) + (1, 0)$);
			\fill[catmca] (bot) rectangle ($(effthe) + (-1, 0)$);
			\fill[catmc] (the) to [out=180, in=-90] (ctlthe) -- (effthe) -- (bool) -- (the);
		\end{pgfonlayer}
	\end{tikzpicture}
	\caption{String diagram for the sentence \textsl{the cat sleeps}.}
	\label{fig:sd-thecatsleeps}
\end{wrapfigure}
Let us define the category $\mathds{1}$ with exactly one object and one arrow:
the identity on that object. It will be shown in grey in the string diagrams
below.
A functor of type $\mathds{1} \to \mC$ is equivalent to choosing an object in
$\mC$, and a natural transformation between two such functors $\tau_{1},
	\tau_{2}$ is exactly an arrow in $\mC$ of type $\tau_{1} \to \tau_{2}$.
Knowing that allows us to represent the type resulting from a sequence of
computations as a sequence of strings whose farthest right represents an object
in $\mC$, that is, a base type.

In the diagram of Figure \ref{fig:sd-thecatsleeps}, each string corresponds to
a functorial effect or type layer applied during parsing.
The base type string $\t$ is at the border of the gray area and is the one of
the uneffectful denotation in $\mC$ while the \emph{functorial} string for
$\f{M}$ introduces the effect for optionality and possible failure of the
computation.
The question of providing rules to compose the string diagrams for parts of the
sentences will be discussed in the next section, as it is related to parsing.

\smallskip

In the end, we will have the need to go from a certain set of strings (the effects that applied) to a single one, through a sequence of handlers, monadic and comonadic rules and so on.
Notice that we never reference the zero-cells and that in particular their colors are purely an artistical touch.

\subsection{Achieving Normal Forms}
We will now provide a set of rewriting rules on string diagrams (written as
equations) which define the set of different possible reductions.

First, Theorem \ref{thm:isotopy} reminds the main result by \cite{joyalGeometryTensorCalculus1991} about string diagrams which shows that our artistic representation of diagrams is correct and does not modify the equation or the rule we are presenting.
\begin{theorem}[Theorem 1.2 \cite{joyalGeometryTensorCalculus1991}]
	\label{thm:isotopy}
	String diagrams equivalent under planar isotopy in the graphical language are equal.
\end{theorem}

A few equations on string diagrams also arise from the commutation of certain
class of diagrams and thus typing judgements.
We consider the \emph{snake} equations are a rewriting of the categorical
diagrams which are the defining properties of an adjunction and the
\emph{(co-)monadic} equations are the string diagrammatic translation of the
properties of unitality and associativity of monads.
These equations (and the reduction rules from Section \ref{subsec:rewrite})
explain all the different reductions that can be made to limit non-determinism
in our parsing and handling strategies.

\subsection{Computing Normal Forms}
Now that we have a set of rules telling us what we can and cannot do in our
model while preserving the equality of the diagrams, we provide a combinatorial
description of our diagrams to help compute the possible equalities between
multiple reductions of a sentence meaning.
In this section we formally describe the data structure we propose, as well as
proving our system of rewriting allows us to compute normal forms for our
diagrams.

\subsubsection{Representing String Diagrams}
We follow \cite{delpeuchNormalizationPlanarString2022} in their combinatorial
description of string diagrams.
We describe a diagram by an ordered set of its $2$-cells (the natural
transformations, including handlers of the diagram) along with the number of
input strings, for each $2$-cell we log the following information:
\begin{itemize}
	\item Its horizontal position: the number of strings that are right of it.
	\item Its type: an array of effects that are the inputs to the natural
	      transformation and an array of effects that are the outputs to the
	      natural transformation.
\end{itemize}
We will then write a diagram $D$ as a tuple of $3$ elements:
$\left( D.N, D.S, D.L \right)$ where $D.N$ is a positive integers representing the
height (or number of nodes) of $D$, $D.S$ is an array for the input strings of $D$ and
where $D.L$ is a function which takes a natural number smaller than $D.N - 1$ and
returns its type as a tuple of arrays
$nat = \left( \dnlg{nat}, \dnin{nat}, \dnout{nat} \right)$.
This gives a naïve algorithm in polynomial time to check if a string diagram is
valid or not.

\smallskip

Because our representation contains strictly more information (without slowing
access by a non-constant factor) than the one it is based on, our
data structure supports the linear and polynomial time algorithms proposed with
the structure by \cite{delpeuchNormalizationPlanarString2022}.
In particular our structure can be normalized in time
$\O\left( n \times \sup_{i} \abs{\dnin{\dlb{D}{i}}} + \abs{\dnout{\dlb{D}{i}}}
	\right)$, which depends on our lexicon but most of the times will be linear
time.

\subsubsection{Equational Reductions}
We are faced a problem when computing reductions using the equations for our diagrams
which is that by definition, an equation is symmetric.
To solve this issue, we only use equations from left to right to reduce as much as
possible our result instead.
Moreover, note that all our reductions are either incompatible or commutative, which
leads to a confluent reduction system, and the well definition of our normal forms.
\begin{theorem}[Confluence]
	\label{thm:confluence}
	Our reduction system is confluent and therefore defines normal forms:
	\begin{enumerate}
		\item Right reductions are confluent and therefore define \emph{right} normal forms for
		      diagrams under the equivalence relation induced by exchange.
		\item Equational reductions are confluent and therefore define \emph{equational}
		      normal forms for diagrams under the equivalence relation induced by exchange.
	\end{enumerate}
\end{theorem}

Before proving the theorem, let us first provide the reduction rules for the
different equations for our description of string diagrams.
\begin{description}
	\item[The Snake Equations]
	      First, let's see when we can apply the equation for $\id_{L}$ to a
	      diagram $D$ which is in \emph{right} normal form, meaning it's been
	      right reduced as much as possible.
	      Suppose we have an adjunction $L \dashv R$.
	      Then we can reduce $D$ along the equation at $i$ if, and only if:
	      \begin{itemize}
		      \item $\dnlg{\dlb{D}{i}} = \dnlg{\dlb{D}{i + 1}} - 1$
		      \item $\dlb{D}{i} = \eta_{L, R}$
		      \item $\dlb{D}{i + 1} = \epsilon_{L, R}$
	      \end{itemize}
	      This comes from the fact that we can't send either $\epsilon$
	      above $\eta$ using right reductions and
	      that there cannot be any natural transformations between the two.
	      Obviously the equation for $\id_{R}$ works the same.
	      Then, the reduction is easy: we simply delete both strings,
	      removing $i$ and $i + 1$ from $D$ and reindexing the other nodes.
	\item[The Monadic Equations] For the monadic equations, we only use
	      the unitality equation as a way to reduce the number of natural
	      transformations, since the goal here is to attain normal forms
	      and not find all possible reductions.
	      We ask that associativity is always used in the direct
	      sense $\mu\left( \mu\left( TT \right),T \right) \to \mu\left(
		      T\mu\left( TT \right) \right)$ so that the algorithm terminates.
	      We use the same convention for the comonadic equations.
	      The validity conditions are as easy to define for the monadic
	      equations as for the \emph{snake} equations when considering
	      diagrams in \emph{right} normal forms.
	      Then, for unitality we simply delete the nodes
	      and for associativity we switch the horizontal
	      positions for $i$ and $i + 1$.
\end{description}

\begin{proof}[Proof of the Confluence Theorem]
	The first point of this theorem is exactly Theorem 4.2
	in \cite{delpeuchNormalizationPlanarString2022}.
	To prove the second part, note that the reduction process terminates as
	we strictly decrease the number of $2$-cells with each reduction.
	Moreover, our claim that the reduction process is confluent is obvious
	from the associativity equation and the fact the other
	equations delete nodes.
	Since right reductions do not modify the equational reductions, and thus
	right reducing an equational normal form yields an equational normal form,
	combining the two systems is done without issue, completing our proof of
	Theorem \ref{thm:confluence}.
\end{proof}

\begin{theorem}[Normalization Complexity]
	\label{thm:norm}
	Reducing a diagram to its normal form is done in polynomial time in
	the number of natural transformations in it.
\end{theorem}
\begin{proof}
	Let's now give an upper bound on the number of reductions.
	Since each reductions either reduces the number of $2$-cells or applies the
	associativity of a monad, we know the number of reductions is linear in the
	number of natural transformations.
	Moreover, since checking if a reduction is possible at height $i$ is done in
	constant time, checking if a reduction is possible at a step is done in
	linear time, rendering the reduction algorithm quadratic in the number of
	natural transformations.
	Since \emph{right} normalizing in linear time before to ensure we get all
	equational reductions and after to complete the reduction is enough,
	we have a polynomial time algorithm.
\end{proof}

\section{Efficient Semantic Parsing}
\label{sec:parsing}
In this section we explain our algorithms and heuristics for efficient semantic
parsing with as little ambiguity as possible, and reducing time complexity
of our parsing strategies.

\subsection{Syntactic-Semantic Parsing}
Using a naïve strategy of type checking on syntax trees yields an exponential
algorithm.
To avoid that, we extend the grammar system used to do the syntactic part of
the parsing to include semantic combination of words.
In this section, we will take the example of a CFG since it suffices to create
our typing combinators,
In Figure \ref{fig:combination-cfg}, we explicit a grammar of combination
modes, based on \cite{bumfordEffectdrivenInterpretationFunctors2025} as it
simplifies the rewriting of our typing judgements in a CFG.

\begin{figure}
	\input{combination-cfg}
	\caption{Possible type combinations in the form of a near CFG. Here, $a, b\in \star_{0}$, $\alpha, \beta, \tau \in \star$ and $\f{F}, \f{L}, \f{R} \in \mF(\mL)$ with $\f{L} \dashv \f{R}$.}
	\label{fig:combination-cfg}
\end{figure}

This grammar works in five major sections:
\begin{enumerate}
	\item We reintroduce the grammar defining the type and effect system.
	\item We introduce a structure for the semantic parse trees and their labels,
	      based on the combination modes from
	      \cite{bumfordEffectdrivenInterpretationFunctors2025}.
	\item We introduce rules for basic type combinations.
	\item We introduce rules for higher-order unary type combinators.
	\item We introduce rules for higher-order binary type combinators.
\end{enumerate}

Each of these combinators can be, up to order, associated with a inference
rule, and, as such, with a higher-order denotation, which explains the actual
effect of the combinator, and are described in
Figure \ref{fig:combinator-denotations}.

\begin{figure}
	\input{combinator-denotations}
	\caption{Denotations describing the effect of the combinators used in the
		grammar describing our combination modes presented in
		Figure \ref{fig:combination-cfg}}
	\label{fig:combinator-denotations}
\end{figure}

The main reason why denotations associated to combinators are needed, is to
properly define how they actually do the combination of denotations.
Those denotations are a direct translation of the judgements defining the
notions of functors, applicatives, monads and thus are not specific to any
denotation system, even though we use lambda-calculus to describe them.
Some are duplicated for a left and right version to account for the fact CFGs
are not actually symmetric in their "input" unlike intuitionistic inference
rules.

This makes us able to compute the actual denotations associated to a sentence
using our formalism, as presented in Figure \ref{fig:parsing-trees}.
Note that the order of combination modes is not actually the same as the one
that would come from the grammar.

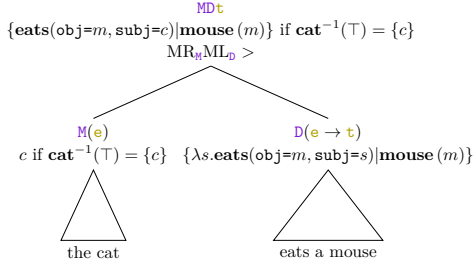
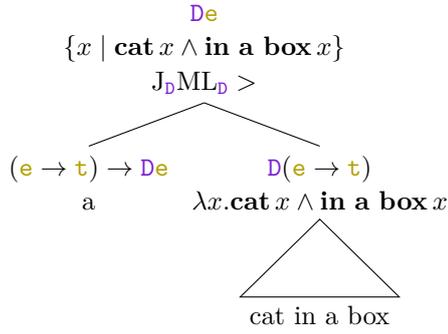
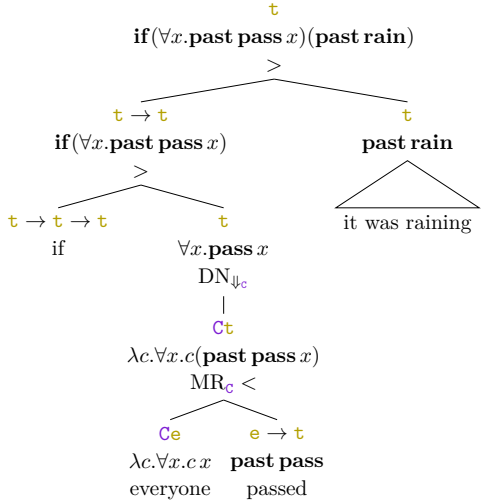
\begin{wrapfigure}[39]{l}{.45\textwidth}
	\centering
	\begin{subfigure}{.45\textwidth}
		\centering
		\resizebox{\textwidth}{!}{\begin{tikzpicture}[every tree node/.style={align=center, anchor=north}, level distance=2.5cm]
				\Tree [
				.{$\f{M}\f{D}\t$ \\ $\left\{\mathbf{eats}(\texttt{obj=}m, \texttt{subj=}c) \middle| \w{mouse}(m)\right\}$ if $\mathbf{cat}^{-1}(\top) = \{c\}$ \\ $\combMR_{\f{M}}\combML_{\f{D}}>$}
				[
				.{$\f{M}(\e)$ \\ $c$ if $\mathbf{cat}^{-1}(\top) = \{c\}$} \edge[roof]; {the cat}
				]
				[
				.{$\f{D}(\e \to \t)$ \\ $\left\{\lambda s. \mathbf{eats}(\texttt{obj=}m, \texttt{subj=}s) \middle| \w{mouse}(m)\right\}$} \edge[roof]; {eats a mouse}
				]
				]
			\end{tikzpicture}}
		\caption{Labelled tree representing the equivalent parsing diagram to
			\ref{fig:parsing-diagram}}
		\label{fig:tree-eats}
	\end{subfigure}

	\begin{subfigure}{.45\textwidth}
		\centering
		\begin{tikzpicture}[every tree node/.style={align=center, anchor=north}, level distance=2cm]
			\Tree [
			.{$\f{D}\e$ \\ $\{x \mid \w{cat} x \land \w{in\ a\ box} x \} $ \\ $\combJ_{\f{D}}\combML_{\f{D}} >$}
			{$(\e \to \t) \to \f{D}\e$ \\ a}
			[ .{$\f{D}(\e \to \t)$ \\ $\lambda x. \w{cat} x \land \w{in\ a\ box} x$} \edge[roof]; {cat in a box} ] ]
		\end{tikzpicture}
		\caption{Labelled tree representing the equivalent parsing diagram to
			\ref{fig:parsing-diagram2}}
		\label{fig:tree-box}
	\end{subfigure}

	\begin{subfigure}{.45\textwidth}
		\centering
		\resizebox{\textwidth}{!}{\begin{tikzpicture}[every tree node/.style={align=center, anchor=north}, level distance=1.75cm]
				\Tree [
				.\node{\t \\ $\mathbf{if}(\forall x. \w{past}\w{pass} x)(\w{past}\mathbf{rain})$ \\ $>$};
				[ .\node{$\t \to \t$ \\ $\mathbf{if}(\forall x.\w{past}\w{pass} x)$ \\ $>$};
				{$\t \to \t \to \t$ \\ if}
				[
				.\node{$\t$\\ $\forall x. \w{pass} x$ \\ $\combDN_{\Downarrow_{\f{C}}}$};
				[ .\node{$\f{C}\t$ \\ $\lambda c.\forall x. c(\w{past}\w{pass} x)$ \\ $\combMR_{\f{C}}<$};
				{$\f{C}\e$ \\ $\lambda c. \forall x. c\, x$ \\ everyone}
				{$\e \to \t$ \\ $\w{past}\mathbf{pass}$ \\ passed}
				]
				]
				]
				[
				.{\t \\ $\w{past}\mathbf{rain}$} \edge[roof]; {it was raining}
				]
				]
			\end{tikzpicture}}
		\caption{Labelled tree representing the equivalent parsing diagram to
			\ref{fig:3dparsing-diagram}}
		\label{fig:tree-rain}
	\end{subfigure}
	\caption{Examples of labelled parse trees for a few sentences.}
	\label{fig:parsing-trees}
\end{wrapfigure}

The reason why will become more apparent when string diagrams for parsing are
introduced in the next section, but simply, this comes from the fact that while
we think of $\combML$ and $\combMR$ as reducing the number of effects on each
side (and this is the correct way to think about those), this is not actually
how its denotation works, they are actually modifying a combination mode via
their denotation.
This formalism gives us the following theorems:

\begin{theorem}
	\label{thm:ptime-parse}
	Parsing of a sentence with combination modes is polynomial in the length of
	the	sentence and the size of the type system and syntax system.
\end{theorem}

\begin{proof}
	Suppose we are given a syntactic generating structure $G_{s}$ along with our
	type combination grammar $G_{\tau}$.
	The syntactico-semantic system $G$ constructed from the product of $G_{s}$ and
	$G_{\tau}$ has size $\abs{G_{s}}\times \abs{G_{\tau}}$.
	Computing membership of a sentence to the language generated by $G$, is then
	in polynomial time if, and only if, finding membership to the language
	generated by $G_{s}$ is done in polynomial time.
	Parsing the sentence is then done in polynomial time in the size of the
	input, $\abs{G_{s}}$ and
	$\abs{G_{\tau}} = \O\left(\abs{\mF\left(\mL\right)} + \abs{\Obj\left(\mC\right)}\right)$.
\end{proof}

\begin{theorem}
	\label{thm:ptime-denot}
	Retrieving a pure denotation for a sentence is polynomial in
	the length of the sentence, given a polynomial time syntactic parsing
	structure and polynomial combinator denotations.
\end{theorem}

\noindent To prove this theorem we need a short lemma on the size of the trees generated
through our structure:
\begin{lemma}
	\label{lem:quad-tree}
	Semantic parsing trees are quadratic in the length of the sentence.
\end{lemma}

\begin{proof}
	Let $m_{\mL}$ be the maximum number of effects created by a word in $\mL$.
	Since at any step $i$ in the parsing, there can never be more than
	$m_{\mL}\times i$ effects borne by the considered inputs, there is no need
	for more than $(2 + c) \times m_{\mL}\times (i + 1) + 1$ combinators where
	$c$ is a constant dependent only on the language.
	Indeed, we will have at most one combinator among
	$\{\combML, \combMR, \combA, \combUR, \combUL\}$ per input effect, at most one
	of $\combJ$ and $\combDN$ per output effects
	($m_{\mL} \times (i + 1)$ at most),	at most a fixed number $c$ of modes
	between $\{\combC, \combEL, \combER\}$ which depends only on the number of
	adjunctions in the language.
	We get the wanted upper bound when adding the \emph{base combinator}.
	Summing the steps for $i$, we get a quadratic upper bound on the number of
	combinators and thus on the tree size.
\end{proof}

\begin{figure}
	\centering
	\begin{tikzpicture}[baseline={([yshift=-.5ex]current bounding box.center)}]
	\path coordinate[dot, label=below:$>$] (m)
	+ (0, 1) coordinate[label=above:$\beta$] (result)
	+ (-1, -1) coordinate[label=below:$\alpha \to \beta$] (phi)
	+ (1, -1) coordinate[label=below:$\alpha$] (x);
	\draw (m) -- (result);
	\draw (m) to[out=180, in=90] (phi);
	\draw (m) to[out=0, in=90] (x);
	\begin{pgfonlayer}{background}
		\fill[catmcb] (result) -- (m) -- ($(m) + (-1, 0)$) |- (result);
		\fill[catmcb] (result) -- (m) to[out=180, in=90] (phi) -- ($(phi) + (-1, 0)$) |- (result);
		\fill[catmc] (result) -- (m) -- ($(m) + (1, 0)$) |- (result);
		\fill[catmc] (result) -- (m) to[out=0, in=90] (x) -- ($(x) + (1, 0)$) |- (result);
		\fill[catmca] (m) to[out=0, in=90] (x) -- (phi) to[out=90, in=180] (m);
	\end{pgfonlayer}
\end{tikzpicture}
\begin{tikzpicture}[baseline={([yshift=-.5ex]current bounding box.center)}, yscale=-1]
	\path coordinate[dot, label=above:$\combML_{\f{F}}$] (m)
	+ (1, 1) coordinate[label=below:$\beta$] (out2)
	+ (-1, 1) coordinate[label=below:$\alpha$] (out)
	+ (-1.5, 1) coordinate[label=below:$\f{F}$] (out1)
	+ (-1, -1) coordinate[label=above:$\alpha$] (phi)
	+ (1, -1) coordinate[label=above:$\beta$] (x);
	\draw (m) to[out=180, in=90] (phi);
	\draw (m) to[out=0, in=90] (x);
	\draw (m) to[out=180, in=-90] (out1);
	\draw (m) to[out=180, in=-90] (out);
	\draw (m) to[out=0, in=-90] (out2);
	\begin{pgfonlayer}{background}
		\fill[catmcc] (m) to[out=180, in=90] (phi) -- ($(phi) + (-1, 0)$) |- (out1) to[out=-90, in=180] (m);
		\fill[catmc] (m) to[out=0, in=-90] (x) -- ($(x) + (1, 0)$) |- (out2) to[out=-90, in=0] (m);
		\fill[catmca] (m) to[out=180, in=90] (phi) -- (x) to[out=90, in=0] (m);
		\fill[catmcb] (m) to[out=180, in=-90] (out1) -- (out) to[out=-90, in=180] (m);
		\fill[catmca] (m) to[out=180, in=-90] (out) -- (out2) to[out=-90, in=0] (m);
	\end{pgfonlayer}
\end{tikzpicture}
\begin{tikzpicture}[baseline={([yshift=-.5ex]current bounding box.center)}]
	\path coordinate[dot, label=below:$\combJ_{\f{F}}$] (m)
	+ (0, 1) coordinate[label=above:$\f{F}$] (result)
	+ (-1, -1) coordinate[label=below:$\f{F}$] (phi)
	+ (1, -1) coordinate[label=below:$\f{F}$] (x);
	\draw (m) -- (result);
	\draw (m) to[out=180, in=90] (phi);
	\draw (m) to[out=0, in=90] (x);
	\begin{pgfonlayer}{background}
		\fill[catmcb] (result) -- (m) -- ($(m) + (-1, 0)$) |- (result);
		\fill[catmcb] (result) -- (m) to[out=180, in=90] (phi) -- ($(phi) + (-1, 0)$) |- (result);
		\fill[catmc] (result) -- (m) -- ($(m) + (1, 0)$) |- (result);
		\fill[catmc] (result) -- (m) to[out=0, in=90] (x) -- ($(x) + (1, 0)$) |- (result);
		\fill[catmca] (m) to[out=0, in=90] (x) -- (phi) to[out=90, in=180] (m);
	\end{pgfonlayer}
\end{tikzpicture}
	\caption{String Diagrammatic Representation of Combinator Modes $>, \combML$ and $\combJ$}
	\label{fig:combinator-sd}
\end{figure}

\begin{wrapfigure}[29]{r}{.45\textwidth}
	\centering
	\includegraphics[width=.45\textwidth]{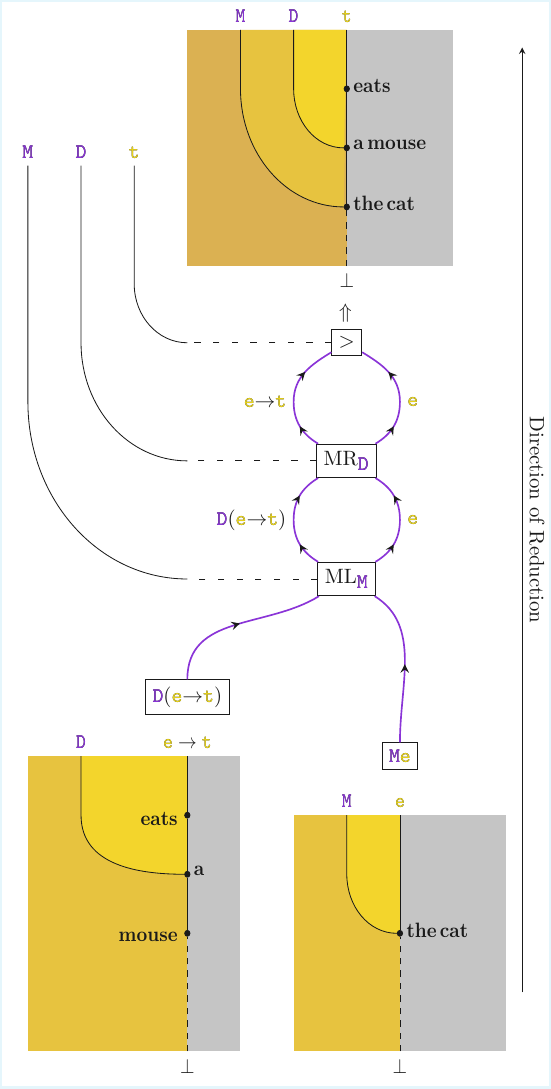}
	\caption{Representation of a parsing diagram for the sentence
		\emph{the cat eats a mouse}.
		See Figure \ref{fig:tree-box} for translation in a parse tree.}
	\label{fig:parsing-diagram}
\end{wrapfigure}

We can now return to the proof of the main result of this section:
\begin{proof}[Proof of Theorem \ref{thm:ptime-denot}]
	From Theorem \ref{thm:ptime-parse} we can retrieve a
	semantic parse tree from a sentence in polynomial time in the input.
	Lemma \ref{lem:quad-tree} states that we have a polynomial number of
	combinator denotations to apply, all done in polynomial time by hypothesis.
	We have already seen that given a denotation, handling all effects and
	reducing effect handling to normal forms can be done in polynomial time.
	The sequencing of these steps yields a polynomial-time algorithm in the
	length of the input sentence.
\end{proof}

While we have gone the assumption that we have a CFG for our language,
any type of polynomial-time structure could work, as long as it is at least
as expressive as a CFG.

The \emph{polynomial time combinators} assumption in Theorem
\ref{thm:ptime-denot} is not a complex assumption, this is for example true for
denotations based on lambda-calculus, with function application being linear in
the number of uses of variables in the function term, which in turn is linear
in the number of terms used to construct the function term and thus of words,
and the different \fmap{} being in polynomial time for the same reason.
This would also be true for denotations inspired by machine learning for
example.

\subsection{Diagrammatical Parsing}
When considering \cite{coeckeMathematicalFoundationsCompositional2010}
way of using string diagrams for syntactic parsing/reductions, we can see
string diagrams as (yet) another way of writing our parsing rules.
They are an expanded rewriting of labelled parsing trees\footnote{Point of view
	which connects this formalism nicely to the one of
	\cite{senturiaAlgebraicStructureMorphosyntax2025}, preserving all their
	results inside our theory.} presented in
\cite{bumfordEffectdrivenInterpretationFunctors2025}, .
In our typed category, we can see our combinators as natural transformations
($2$-cells): then we can see the different sets of combinators as different
arity natural transformations.
Combinators $>$, $\combML_{\f{F}}$ and $\combJ_{\f{F}}$ are represented in
Figure \ref{fig:combinator-sd}.
The coloring of the regions is purely for artistic rendition and will not be
used for larger diagrams.

\begin{wrapfigure}{r}{.45\textwidth}
	\centering
	\includegraphics[width=.45\textwidth]{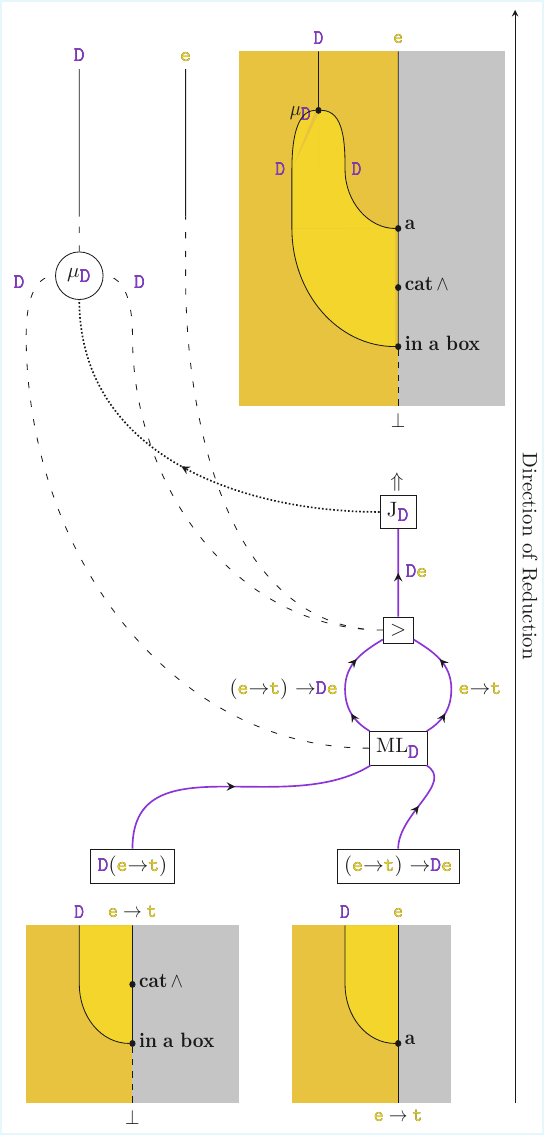}
	\caption{Example of a parsing diagram for the phrase
		\emph{a cat in a box}, presenting the integration of unary combinators
		inside the connector line. See Figure \ref{fig:tree-box} for translation in
		a parse tree.}
	\label{fig:parsing-diagram2}
\end{wrapfigure}
Understanding the diagrams could be thinking of them on an orthogonal plane to
the ones of Section \ref{sec:nondet}: we could use the syntactic version of the
diagrams to model our parsing, according to the rules in Figure
\ref{fig:combination-cfg}, and then combine the diagrams as shown in Figure
\ref{fig:parsing-diagram}, which highlights the \emph{orthogonal} components.
In this diagram we exactly see the sequence of combinations play out on the
types of the words, and thus we also see what exact \emph{stitch} would
be needed to construct the effect diagram.
Here we talk about \emph{stitches} because, in a sense, we use $2$-cells
to do braiding-like operations on the strings, and don't actually allow for
braiding inside the diagrammatic computation, leading to the intervention of
outside tools (combinators) which serve as \emph{knitting needles}.
To better understand what happens in those parsing diagrams, Figure
\ref{fig:parsing-trees} provides the translations in labelled trees of the
parsing diagrams of Figures \ref{fig:parsing-diagram},
\ref{fig:parsing-diagram2} and \ref{fig:3dparsing-diagram}.

For the combinators $\combJ$, $\combDN$ and $\combC$, which are applied to
reduce the number of effects inside a denotation, it might seem less obvious
how to include them.
Applying them to the actual \emph{parsing} part of the diagram is done
in the exact same way as in the CFG: we just add them where needed, and they
will appear in the resulting denotation as a form of forced handling, in a
sense, as shown in the result of Figure \ref{fig:parsing-diagram2}.
It is interesting to note that the resulting diagram representing
the sentence can visually be found in the connection strings that arise from
the combinators.

\smallskip

Categorically, we start from a meaning category $\mC$, our typing category, and
take it as our grammatical category.
This is a form of extension on the monoidal version by
\cite{coeckeMathematicalFoundationsCompositional2010} and
\cite{toumiHigherOrderDisCoCatPeirceLambekMontague2023}, as it is seemingly a
typed version, where we change the pregroup category for the typing category,
taken with a product for representation of the English grammar representation,
to accommodate for syntactic typing on top of semantic typing if it does not
already encompass it.
We have a first plane of string diagrams in the category
$\mC$ - our string diagrams for effect handling, as in Section
\ref{sec:nondet} - and the second \emph{orthogonal} plane of string diagrams
on a larger category, with formal endofunctors labelled by the types in our
typing category $\bar{\mC}$ and formal natural transformations for the
combinators defined in Figures \ref{fig:combination-cfg} and
\ref{fig:combinator-denotations}.
\begin{wrapfigure}{r}{.45\textwidth}
	\centering
	\begin{tikzpicture}
		\node (fig) at (0, 0) {\includegraphics[width=.4\textwidth]{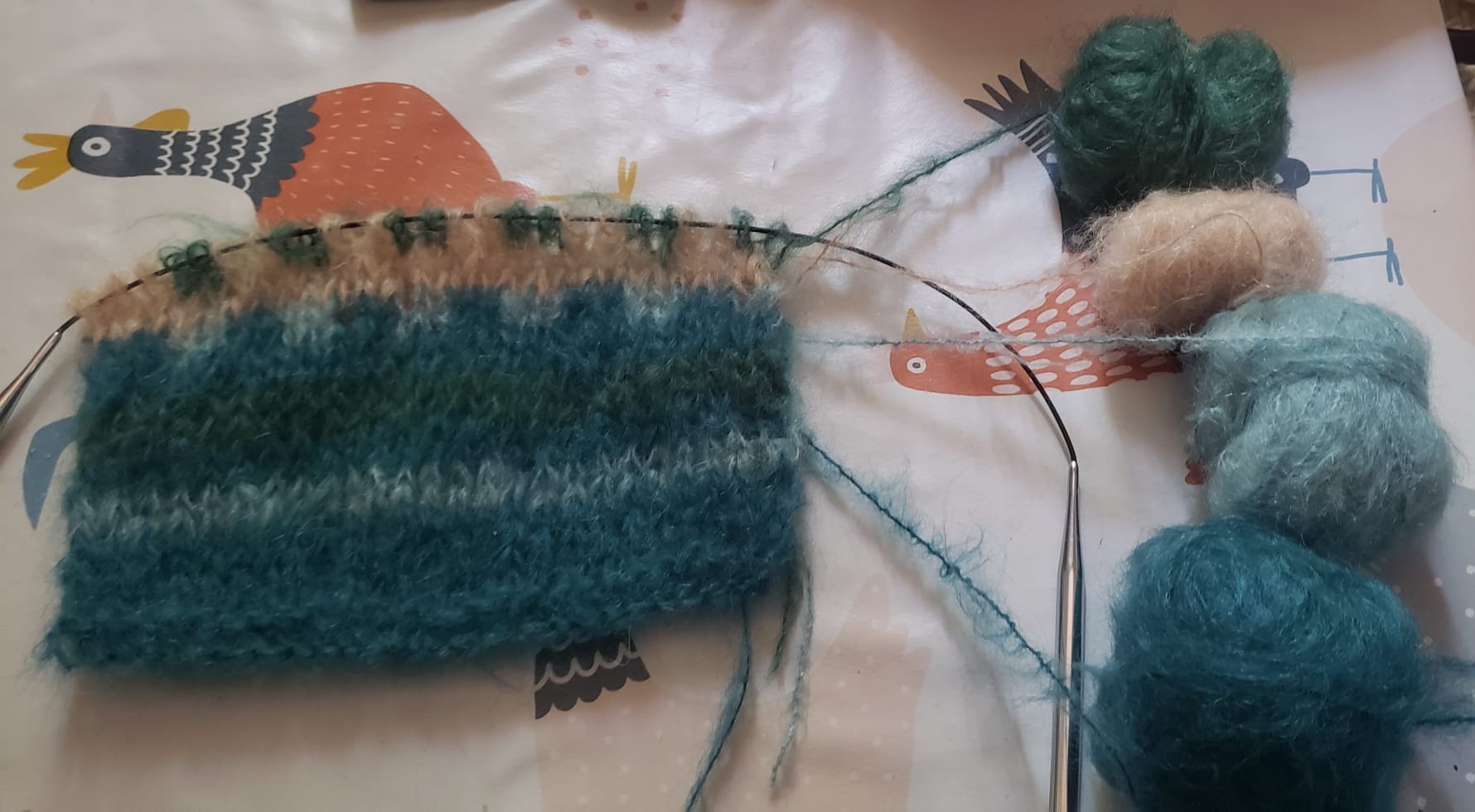}};
		\draw[->] ($(fig.south west) + (-.1, -.1)$) -- node[anchor=east] {\rotatebox{90}{Direction of Reduction}} ($(fig.north west) + (-.1, .1)$);
	\end{tikzpicture}
	\caption{Example of a \emph{Jacquard} knitwork. Photography and work courtesy
		of the author's mother.}
	\label{fig:knitting-example}
\end{wrapfigure}

The category in which we consider the second-axis string diagrams does not have
a meaning in our compositional semantics theory, and to be more precise, we
should talk about $1$-cells and $2$-cells instead of endofunctors and natural
transformations, to keep in the idea that this is really just a diagrammatic
way of computing and presenting the operations that are put to work during
semantic parsing.

The main theoretical reason why this point of view of diagrammatic parsing is
useful will be clear when looking at the rewriting rules and the normal forms
they induce, because, as stated in Theorem \ref{thm:norm}, string
diagrams make it easy to compute normal forms when provided with a confluent
reduction system.
However, the just as useful graphical interpretation of string diagrams as
easy to read expanded labelled parsing trees.
Using orthogonal planes to visualise this interpretation cannot be well
presented in a 3D space, and even less so on a page, so we suggest an
interpretation based on actual strings:
Suppose you're knitting a rainbow scarf.

\begin{wrapfigure}[22]{l}{.5\textwidth}
	\centering
	\includegraphics[width=.5\textwidth]{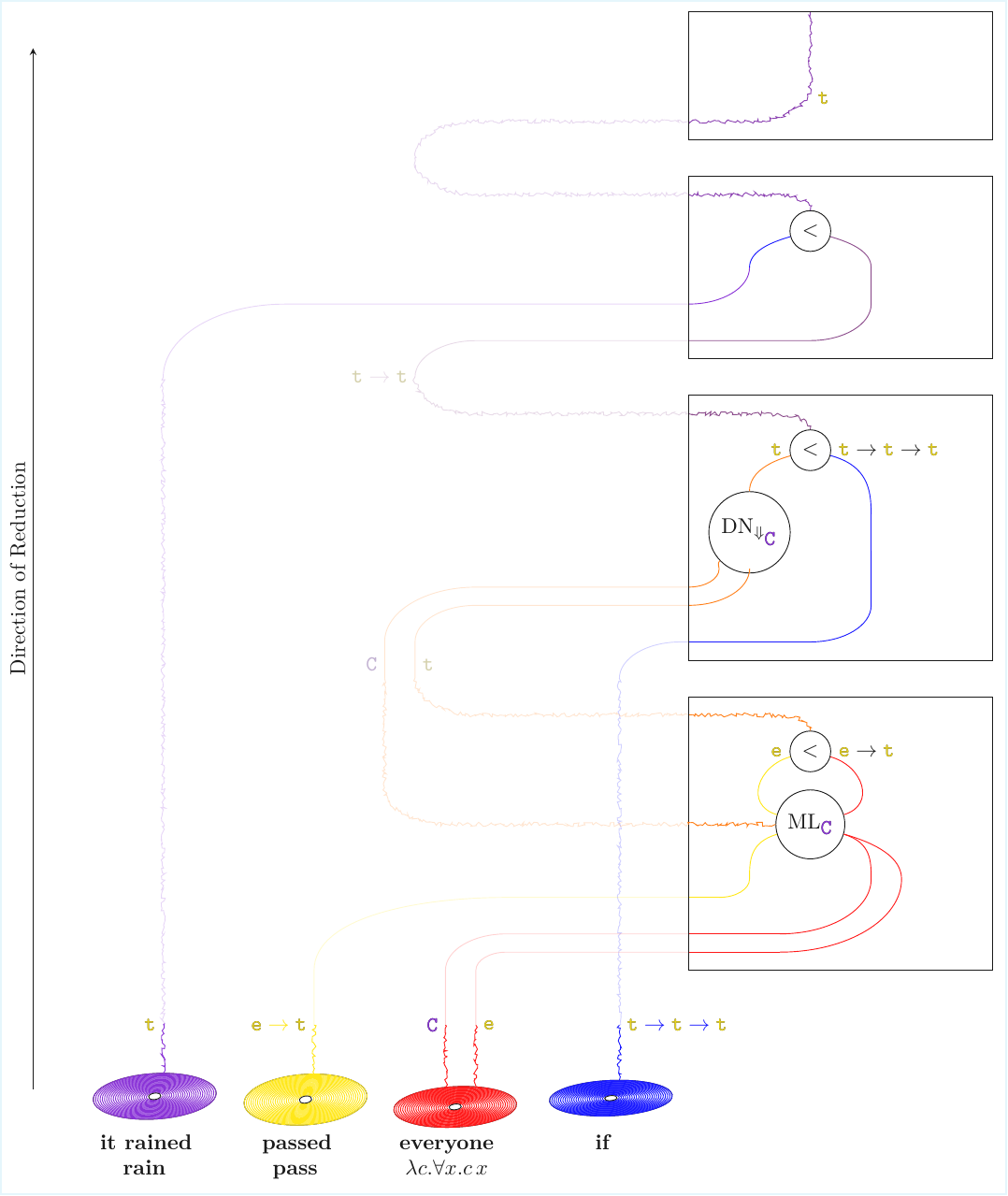}
	\caption{Knitting-like representation of the diagrammatic parsing of a sentence. See Figure \ref{fig:tree-rain} for the translation in a parse tree}
	\label{fig:3dparsing-diagram}
\end{wrapfigure}

You have multiple threads (the different words) of the different colours (their
types and effects) you're using to knit the scarf.
When you decide to change the color, you take the different threads you have
been using, and mix them up.
You can create a new colour\footnote{This is not how wool works, but
	one can also imagine a pointillist-like way of drawing using multiple
	coloured lines that superimpose on each other, or a marching band's multiple
	instruments playing either in harmony or in disharmony and changing that
	during a score.} thread from two (that's the base combinators).
Creating a thicker one from two of the same colour is the result of the
applicative mode and the monadic join.
$\fmap$ puts aside a thread until a later step, the monadic unit adds a new
thread to the pattern, and the co-unit and closure operators cut a thread which
will no longer be used.
Changing a thread by cutting it and making a knot at another point is what the
eject combinators do.

This more tangible representation can be seen in Figure
\ref{fig:3dparsing-diagram}.
The sections in the rectangle represent what happens when considering our
combination step as implementing patterns inside a knitwork, as seen in
Figure~\ref{fig:knitting-example}.
The different patterns provide, in order, a visual representation of the
different ways one can combine two strings, i.e., two types and thus two
denotations.
The sections outside of the rectangle are the strings of yarn not currently
being used to make a pattern.

\subsection{Rewriting Rules}
\label{subsec:rewrite}
In this section we study reductions for our diagrams that allows us
to improve our time complexity by reducing the size of the grammar.
This is done by looking at equations on sequences of combinators.
In the worst case, there is no improvement in big o notation in the size of the
sentence, but there is no loss.

\noindent Consider the case where we have the two arguments of our parsing step of
type $\f{F}\tau$ and $\f{G}\tau'$.
In that case we could either get a result with effects $\f{F}\f{G}$ or
with effects $\f{G}\f{F}$.
If those effects happen to be equal, which trivially will be the case when one
of the effects is external (the plural or islands functors for example), the
order of application does not matter and we choose to get the effect on the
left side of the combinator first: $\combML_{\f{F}}\combMR_{\f{G}}$ over
$\combMR_{\f{F}}\combML_{\f{F}}$.

\noindent There are sequence of modes that clearly encompass other ones
the grammar notation for ease of explanation.
One should not use the unit of a functor after using $\combML$ or $\combMR$, as
that adds void semantics.
Same things can be said for certain other derivations containing the lowering
and co-unit combinators since they could in theory be applied at many points
inside the derivation.

\noindent We use $\combDN$ when we have not used any of the following, in all
derivations:
\let\mcolsep=\multicolsep
\setlength{\multicolsep}{.4\mcolsep}
\begin{multicols}{2}
	\begin{itemize}
		\item $m_{\f{F}}, \combDN, m_{\f{F}}$ where
		      $m \in \{\combMR, \combML\}$
		\item $\combML_{\f{F}}, \combDN, \combMR_{\f{F}}$
		\item $\combA_{\f{F}}, \combDN, \combMR_{\f{F}}$
		\item $\combML_{\f{F}}, \combDN, \combA_{\f{F}}$
		\item $\combC$
	\end{itemize}
\end{multicols}
\noindent We use $\combJ$ if we have not used any of the following,
for $j \in \{\epsilon, \combJ_{\f{F}}\}$
\begin{multicols}{2}
	\begin{itemize}
		\item $\left\{m_{\f{F}}, j, m_{\f{F}}\right\}$ where
		      $m \in \{\combMR, \combML\}$
		\item $\combML_{\f{F}}, j, \combMR_{\f{f}}$
		\item $\combA_{\f{F}}, j, \combMR_{\f{F}}$,
		\item $\combML_{\f{F}}, j, \combA_{\f{F}}$
		\item $k, \combC$ for $k \in \{\epsilon, \combA_{\f{F}}\}$
		\item If $\f{F}$ is commutative as a monad:
		      \begin{itemize}
			      \item $\combMR_{\f{F}}, \combA_{\f{F}}$
			      \item $\combA_{\f{F}}, \combML_{\f{F}}$
			      \item $\combMR_{\f{F}}, j, \combML_{\f{F}}$
			      \item $\combA_{\f{F}}, j, \combA_{\f{F}}$
		      \end{itemize}
	\end{itemize}
\end{multicols}

\begin{theorem}
	The rules proposed above yield equivalent results.
\end{theorem}

\begin{proof}
	The rules about not using combinators $\combUL$ and $\combUR$ come from the
	notion of handling and granting termination and decidability to our system.
	The rules about adding $\combJ$ and $\combDN$ after moving two of the same
	effect from the same side (i.e. $\combML \combML$ or $\combMR\combMR$) are
	normalization along Theorem \ref{thm:isotopy}: the only reason to keep two of
	the same effects and not join them is to at some point	have something get in
	between the two.
	Joining and closure should then be done at earliest point in parsing where it
	can be done, and that is equivalent to later points because of Theorem
	\ref{thm:isotopy}.
	The last set of rules follows from the following: we should not use $\combJ
		\combML \combMR$ instead of $\combA$, as those are equivalent because of the
	equation defining them.
	The same thing goes for the other two, as we should use the units of monads
	over applicative rules and \fmap.
\end{proof}

The reductions described above amount to equational reductions for the string
diagrams, as they are equivalent to specific sequences of $2$-cells.
This leads to the same algorithms developed in Section \ref{sec:nondet} being
usable here: we just have a new improved version of Theorem
\ref{thm:confluence}: computing two different normal forms along the
tensor product of our reduction schemes, which amounts to computing a larger
normal form.
Theorem \ref{thm:norm} still stands with the improved system and thus, proving
two parses are equal can be done in polynomial time.
Moreover, considering the possible normal forms of syntactic reductions
or denotational reductions adds ways to reduce our diagrams to normal forms.

\section{Conclusion}
The functional programming approach developed in
\cite{bumfordEffectdrivenInterpretationFunctors2025} allows for increased
expressiveness in the choice of denotations, especially from a purely
theoretical point of view.
In this paper we have successfully proven that such an approach is well-founded
theoretically, but also that it doesn't come at the cost of comprehensibility
or efficiency.
Given the entirely theoretical denotations for common language objects (think
of $\mathbf{cat} = \mathbf{cat}$ as a definition), our methods might give
enough latitude to semanticists to imagine more precise denotations without the
cost of heavy statistical analyses, or at least, give tools to expand on them.
Deriving our formalism from a theory necessitates only to understand what base
combinators exist for the model: we build upon a basic semantic theory to
increase its expressiveness.

Moreover, while our methods for implementing a type and effect system have been
applied to natural language semantics, they could be applied in any language
with purely compositional semantics.
Of course, improvements can be made, in particular around the unorthodox use of
effects to define what we have called higher-order constructs and scope
islands, but also in integrating the theory in more complicated models of
denotations, such as the ones learned through a neural network for example.

\appendix
\bibliography{tdparse.bib}

\section{Presenting a Language}
In this appendix, we provide tables (\ref{fig:lexicon} and \ref{fig:functors}) describing the
modeling of a subset of the English language in our formalism.

\begin{figure}[h!]
	\centering
	\begin{subfigure}{.9\textwidth}
		\centering
		\setcellgapes{3pt}
\makegapedcells
\begin{NiceTabular}{>{\bf}LLL}
	Expression & \rm Type & \lambda\text{-Term} \\
	\word{planet}{\e\to\t}{\lambda x. \w{planet} x}{common nouns}
	\word{carnivorous}{\left( \e \to \t \right)}{\lambda x. \w{carnivorous}x}{predicative adjectives}
	\word{skillful}{\left( \e \to \t \right) \to \left( \e \to \t \right)}{\lambda p. \lambda x. px \land \w{skillful} x}{predicate modifier adjectives}
	\word{Jupiter}{\e}{{\bf j}\in \Var}{}
	\word{sleep}{\e \to \t}{\lambda x. \w{sleep} x}{}
	\word{chase}{\e \to \e \to \t}{\lambda o. \lambda s. \mathbf{chase}\left( o \right)\left( s \right)}{}
	\word{be}{\left( \e \to \t \right) \to \e \to \t}{\lambda p. \lambda x. px}{}
	\word{it}{\f{G}\e}{\lambda g. g_{0}}{}
	\word{the}{\left( \e \to \t \right) \to \f{M}\e}{\lambda p. x \text{ if } p^{-1}\left( \top \right) = \{x\} \text{ else } \#}{}
	\word{a}{\left( \e \to \t \right) \to \f{D}\e}{\lambda p. \lambda s. \left\{ \scalar{x, x \ppl s}\suchthat p x\right\}}{}
	\word{no}{\left( \e \to \t \right) \to \f{C}\e}{\lambda p. \lambda c. \lnot \exists x. p x \land c\, x}{}
	\word{\cdot\w{, a} \cdot}{\e \to \left(\e \to \t\right) \to \f{W}\e}{\lambda x. \lambda p. \scalar{x, p x}}{}
	\CodeAfter
	\begin{tikzpicture}
		\draw[double] (1|-2) -- (4|-2);
		\foreach \r in {4,6,8} {\draw (1|-\r) -- (4|-\r);}
		\foreach \r in {9,...,16} {\draw (1|-\r) -- (4|-\r);}
	\end{tikzpicture}
\end{NiceTabular}

		\caption{Lexicon for a subset of the English language}
		\label{fig:lexicon}
	\end{subfigure}

	\medskip

	\begin{subfigure}{\textwidth}
		\centering
		\resizebox{\textwidth}{!}{%
			\def\arraystretch{1.3}
\setcellgapes{3pt}
\makegapedcells
\begin{NiceTabular}{LLc}
	\rm Constructor                                                            & \fmap                                                                                                                                                             & Interpretation \\
	\f{G}\left( \tau \right) = \r \to \tau                                     & \f{G}\phi\left( x \right) = \lambda r. \phi \left(x r\right)                                                                                                      & Read           \\
	\f{W}\left( \tau \right) = \tau \times \t                                  & \f{W}\phi\left( \scalar{a, p} \right) = \scalar{\phi a, p}                                                                                                        & Write          \\
	\f{S}\left( \tau \right) = \{ \tau \}                                      & \f{S}\phi\left( \left\{ x \right\} \right) = \left\{ \phi(x) \right\}                                                                                             & Powerset       \\
	\f{C}\left( \tau \right) = \left( \tau \to \t \right) \to \t               & \f{C}\phi\left( x \right) = \lambda c. x\left( \lambda a. c \left( \phi a \right) \right)                                                                         & Continuation   \\
	\f{D}\left( \tau \right) = \ty{s} \to \f{S}\left(\tau \times \ty{s}\right) & \f{D}\phi\left( \lambda s. \left\{ \scalar{x, x \ppl s} \suchthat p x \right\} \right) = \lambda s. \left\{ \scalar{\phi x, \phi x \ppl s} \suchthat p x \right\} & State                \\
	\f{M}\left( \tau \right) = \tau + \bot                                     & \f{M}\phi\left( x \right) = \begin{cases}
		                                                                                                         \phi\left( x \right) & \text{if } \cont x: \tau \\
		                                                                                                         \#                   & \text{if } \cont x: \#
	                                                                                                         \end{cases}                                                                                                                & Maybe                     \\
	\CodeAfter
	\begin{tikzpicture}
		\draw[double] (1|-2) -- (4|-2);
		\foreach \r in {3,...,7} {\draw (1|-\r) -- (4|-\r);}
	\end{tikzpicture}
\end{NiceTabular}

		}
		\caption{Definition of a few functors, with their map on functions}
		\label{fig:functors}
	\end{subfigure}
	\caption{Presentation of a lambda-calculus lexicon for the English language}
\end{figure}

\end{document}